\documentclass{article}

\title{\textsc{Minimax Regret for Bandit Convex Optimisation of Ridge Functions}}
\author{Tor Lattimore}

\usepackage{xcolor}
\definecolor{dkblue}{cmyk}{1,.54,.04,.19} 

\usepackage{hyperref} 

\hypersetup{
  bookmarks=true,         
    unicode=false,          
    pdftoolbar=true,        
    pdfmenubar=true,        
    pdffitwindow=false,     
    pdfstartview={FitH},    
    pdftitle={Minimax Regret for Bandit Convex Optimisation of Ridge Functions},    
    pdfauthor={Tor Lattimore},     
    pdfsubject={Bandits},   
    pdfcreator={pdflatex},   
    pdfproducer={Producer}, 
    pdfkeywords={bandits} {online learning} {machine learning}, 
    pdfnewwindow=true,      
    colorlinks=true,        
    linkcolor=black,       
    citecolor=dkblue,       
    filecolor=dkblue,       
    urlcolor=dkblue,        
}

\usepackage{nicefrac}
\usepackage{floatrow}
\usepackage{amsmath}
\usepackage{amsthm}
\usepackage{listings}
\usepackage{tikz}
\usepackage{amssymb}
\usepackage[boxed]{algorithm}
\usepackage{bm}
\usepackage[capitalise]{cleveref}
\usepackage[bf]{caption}
\usepackage{graphicx}
\theoremstyle{plain}
\newtheorem{theorem}{Theorem}

\newtheorem{lemma}[theorem]{Lemma}
\newtheorem{proposition}[theorem]{Proposition}

\theoremstyle{definition}

\theoremstyle{remark}

\usepackage{natbib}

\newcommand{\R}{\mathbb R}

\newcommand{\argmin}{\operatornamewithlimits{arg\,min}}

\newcommand{\ip}[1]{\langle #1 \rangle}
\newcommand{\Reg}{\mathfrak{R}}

\newcommand{\norm}[1]{\Vert #1 \Vert}
\newcommand{\E}{\mathbb E}
\newcommand{\Var}{\mathbb V}
\newcommand{\cE}{\mathcal E}
\newcommand{\cK}{\mathcal K}
\newcommand{\cH}{\mathcal H}
\newcommand{\cF}{\mathcal F}
\newcommand{\cL}{\mathcal L}
\newcommand{\BS}{\mathbb S}
\newcommand{\zeros}{ \bm 0}
\newcommand{\bbP}{\mathbb P}

\newcommand{\const}{\operatorname{const}}

\newcommand{\conv}{\operatorname{conv}}
\newcommand{\diam}{\operatorname{diam}}
\renewcommand{\d}[1]{\operatorname{d}\!#1}
\newcommand{\abovelabel}[1]{\stackrel{\raisebox{0.5mm}{\textrm{\tiny \color{red} #1}}}}

\theoremstyle{definition}
\newtheorem{fact}[theorem]{Fact}

\begin{document}

\maketitle

\begin{abstract}
We analyse adversarial bandit convex optimisation with an adversary that is restricted to playing functions
of the form $f_t(x) = g_t(\langle x, \theta\rangle)$ for convex $g_t : \mathbb R \to \mathbb R$ and unknown $\theta \in \mathbb R^d$ that is homogeneous over time.
We provide a short information-theoretic proof that the minimax regret is at most $O(d \sqrt{n} \log(n \operatorname{diam}(\mathcal K)))$ 
where $n$ is the number of interactions, $d$ the dimension and $\operatorname{diam}(\mathcal K)$ is the diameter of the constraint set.
\end{abstract}

\section{Introduction}
Let $\cK \subset \R^d$ be a convex body (non-empty interior, compact, convex).
A game proceeds over $n$ rounds. At the start of the game, an adversary secretly chooses a vector $\theta \in \BS^{d-1} = \{x \in \R^d : \norm{x} = 1\}$ 
and sequence $(f_t)_{t=1}^n$ such that for all $t$, $f_t : \cK \to \R$ is a function that is:
\begin{enumerate}
\item[(a)] convex; and
\item[(b)] bounded: $f_t(x) \in [0,1]$ for all $x \in \cK$; and
\item[(c)] Lipschitz: $f_t(x) - f_t(y) \leq \norm{x - y}$ for all $x, y \in \cK$; and
\item[(d)] a ridge function: $f_t(x) = g_t(\ip{x, \theta})$ for some $g_t : \R \to \R$.
\end{enumerate}
The learner then sequentially chooses $(x_t)_{t=1}^n$ with $x_t \in \cK$ and observes $f_t(x_t)$, which means that $x_t$
should only depend on the previous actions $(x_s)_{s=1}^{t-1}$, observed losses $(f_s(x_s))_{s=1}^{t-1}$ and possibly a source of randomness.
The minimax regret is
\begin{align*}
\Reg_n = \inf_{\textrm{policy}} \sup_{\textrm{adversary}} \max_{x \in \cK} \E\left[\sum_{t=1}^n f_t(x_t) - f_t(x)\right]\,,
\end{align*}
where the infimum is over all policies, the supremum is over the choices of the adversary subject to the constraints (a)--(d) above, and the expectation
integrates over the randomness in the policy.
Note that the direction $\theta$ of the ridge is \textit{not} known to the learner, but does not change with time. 

\begin{fact}\label{fact}
Suppose that $f$ satisfies (a)--(d) above and $f(x) = g(\ip{x, \theta})$ with $\theta \in \BS^{d-1}$.
Then $g$ is convex and Lipschitz on the closed interval $\{\ip{x, \theta} : x \in \cK\}$.
\end{fact}

A proof is given in \cref{sec:fact}.
A proof of the following theorem is our main contribution.

\begin{theorem}\label{thm:main}
Assume that $\BS^{d-1} \subset \cK$. Then
\begin{align*}
\Reg_n \leq \const d \sqrt{n} \log(n \diam(\cK))\,,
\end{align*}
where $\const$ is a universal constant and $\diam(\cK) = \max_{x, y \in \cK} \norm{x - y}$. 
\end{theorem}

Since the dependence on $\diam(\cK)$ is only logarithmic, the Lipschitz assumption can be relaxed entirely 
by scaling and restricting the domain of $\cK$ as explained by \cite{Lat20-cvx} and \cite{BLE17}.

\paragraph{Related work}
Our setting is a subset of bandit convex optimisation with the adversary restricted to ridge
functions. There is a long line of work without this restriction and correspondingly worse regret bounds.
The study of bandit convex optimisation was initiated by \cite{Kle04} and \cite{FK05}, who introduced simple algorithms
based on gradient descent with importance-weighted gradient estimates. Although these approaches are simple, they suffer
from provably suboptimal dependence on the regret \citep{HPGySz16:BCO}.

A major open question was whether or not $\tilde O(n^{1/2})$ regret is possible, which was closed affirmatively in dimension $1$ by
\cite{BDKP15} and in higher dimensions by \cite{HL16} and \cite{BE18}. The regret of the former algorithm depends
exponentially on the dimension, while in the latter the dependence is polynomial. The best known upper bound on the minimax regret for
bandit convex optimisation is $\tilde O(d^{2.5} \sqrt{n})$ by \cite{Lat20-cvx}.

Our analysis is based on the information-theoretic arguments introduced by \cite{RV14} and
used for the analysis of convex bandits by \cite{BDKP15,BE18} and \cite{Lat20-cvx}.
All these works rely on minimax duality and consequently do not yield
efficient algorithms. This is also true of the result presented here. 

The only polynomial time algorithm with $\tilde O(\sqrt{n})$ regret for the general case is 
by \cite{BLE17}. Although a theoretical breakthrough, the dependence of this algorithm's regret on the dimension 
is $\tilde O(d^{10.5})$ and practically speaking the algorithm is not implementable except when the dimension is very small.
All genuinely practical algorithms for adversarial bandit convex optimisation are gradient methods with importance-weighted gradient estimates.
These algorithms have optimal dependence on the horizon for strongly convex functions \citep{HL14,Ito20} but otherwise not
\citep{Kle04,FK05,Sah11,HPGySz16:BCO}.

As far as we know, the setting of the present paper has not been considered before.
\cite{SNNJ21} tackle the case where $f_t(x) = g_t(h_t(x))$ with $h_t : \R^d \to \R$ a function that is known to the learner.
By choosing $f_t(x) = \ip{x, \theta} + \eta_t \in [0,1]$ for some (adversarial) noise $(\eta_t)_{t=1}^n$, 
our setting subsumes an interesting version of the stochastic linear setting, with the restriction being that the noise is homogeneous and bounded. 
The standard lower bound for this setting is $\Omega(d \sqrt{n})$ \citep{DHK08}, but
with the assumptions required here and by taking $(\eta_t)_{t=1}^n$ to be truncated Gaussian, 
naively this construction would yield a lower bound of $\Omega(d \sqrt{n / \log(n)})$. Nevertheless, this shows that the new result is optimal up to logarithmic factors.
Note that our setting does not subsume the adversarial linear setting because the direction of the ridge is fixed.
An obvious open question is whether or not this can be relaxed.

\paragraph{Notation}
The unit sphere is
$\BS^{d-1} = \{x \in \R^d : \norm{x} = 1\}$. The support function of a set $A \subset \R^d$ is $h_A(\theta) = \sup_{x \in A} \ip{x, \theta}$.
Given a convex function $f : \cK \to \R$, its minimum is $f_\star = \min_{x \in \cK} f(x)$.
The expectation and variance of a random variable $X$ are denoted by $\E[X]$ and $\Var[X]$ respectively. The underlying probability
measure will always be obvious from the context.
All probability measures are defined over the Borel $\sigma$-algebra on the corresponding space.
Given a positive definite matrix $\Sigma$, let $\norm{x}_\Sigma = \sqrt{x^\top \Sigma x}$.
The convex hull of a set $A \subset \R^d$ is $\conv(A)$.

\section{Distributions with high projected variance}\label{sec:var}

Let $A \subset \R^d$ be compact and recall that for $\theta \in \BS^{d-1}$, $h_A(\theta) + h_A(-\theta)$ is the width of $A$ in direction $\theta$. The next lemma
asserts the existence of a distribution on $A$ for which the standard deviation in every direction is nearly
as large as the width.

\begin{lemma}\label{lem:kw}
For any compact set $A \subset \R^d$ there exists random element $X$ supported on $A$ such that
for all
$\theta \in \BS^{d-1}$,
\begin{align*}
h_A(\theta) + h_A(-\theta) \leq 2 \sqrt{d\Var[\ip{\theta, X}]}\,.
\end{align*}
\end{lemma}

\begin{proof}
Assume without loss of generality that the affine hull of $A$ spans $\R^d$. If not, you may work in a suitable affine space.
There exists a probability measure $\pi$ on $A$ such that
with $\mu = \int_A x \d{\pi}(x)$ and $\Sigma = \int_A (x - \mu)(x - \mu)^\top \d{\pi}(x)$, the ellipsoid
$\cE = \{x :\norm{x - \mu}^2_{\Sigma^{-1}} \leq d\}$ is the minimum volume enclosing ellipsoid of $A$ \citep[Corollary 2.11]{Tod16}.
Therefore, when $X$ has law $\pi$,
\begin{align*}
h_A(\theta) + h_A(-\theta)
&= \sup_{x, y \in A} \ip{x - \mu, \theta} + \ip{\mu - y, \theta} \\
&\leq 2\sup_{x \in A} \norm{x - \mu}_{\Sigma^{-1}} \norm{\theta}_\Sigma \\
&\leq 2\sqrt{d \norm{\theta}_\Sigma^2} \\
&= 2\sqrt{d \Var[\ip{X, \theta}]}
\end{align*}
\end{proof}

\section{Information ratio}\label{sec:info}
Let us first recall the main tool, which upper bounds the minimax regret in terms of the information ratio.
The following theorem is a combination of Theorem 2 and Lemma 4 by \cite{Lat20-cvx}, which are proven
by extending the machinery developed by \cite{RV14,BDKP15} and \cite{BE18}.
We outline the key differences in \cref{sec:info-app}.

\begin{theorem}\label{thm:info}
Given $\theta \in \BS^{d-1}$, let
$\cF_\theta$ be the space of all functions satisfying (a)--(d) in the introduction and $\cF = \cup_{\theta \in \BS^{d-1}} \cF_\theta$.
Suppose that $\alpha, \beta \geq 0$ are reals such that
for all $\bar f \in \conv(\cF)$ there exist probability measures $\pi_1,\ldots, \pi_m$ on $\cK$ for which
\begin{align*}
\sup_{f \in \cF} \min_{1 \leq k \leq m} \left(\int_\cK \bar f \d{\pi_k} - f_\star - \sqrt{\beta \int_\cK (\bar f - f)^2 \d{\pi_k}}\right) \leq \alpha\,.
\end{align*}
Then $\Reg_n \leq \const(1 + n\alpha + \sqrt{dnm \beta \log(n \diam(\cK))})$ with $\const$ a universal constant.
\end{theorem}

\vspace{-0.5cm}
\begin{center}
\begin{tikzpicture}
\node[fill=black!5!white,draw,text width=120mm,inner sep=2.5mm,align=justify] (a) at (0,0) {
\begin{minipage}{120mm}
\textbf{Important remark}\,\, A previous version of this manuscript claimed the above theorem even when the ridge $\theta$ was time dependent, meaning that
$f_t(x) = g_t(\ip{x, \theta_t})$ for some adversarial sequence $(\theta_t)_{t=1}^n$.
Unfortunately this result does \textit{not} follow
from the standard machinery. The problem is that $\cF$ is not closed under convex combinations. Why this causes a problem
will be apparent when reading \cref{sec:info-app}. New ideas are needed to understand whether or not \cref{thm:info} might continue to hold
for time varying ridge directions.
\end{minipage}
};
\node[anchor=east,inner sep=0pt,xshift=-0.1cm] at (a.north east) {\scalebox{1.7}{$\bm{***}$}};
\end{tikzpicture}
\end{center}

Combining the following proposition with \cref{thm:info} yields \cref{thm:main}.

\begin{proposition}\label{prop:main}
The conditions of \cref{thm:info} hold with 
\begin{align*}
\alpha &= (2 + 16 \sqrt{d})/n &
\beta &= 2^9 d &
m &\leq 2 + \log_2(n^2 \diam(\cK))\,.
\end{align*}
\end{proposition}

\section{Proof of Proposition~\ref{prop:main}}\label{sec:main}

The proof relies on the following simple lemma.

\begin{lemma}\label{lem:var}
Suppose that $a \leq b$ and $g : [a, b] \to \R$ is convex and continuous.
Let $X$ be a random variable supported on $[a,b]$. Then, for any $\alpha > g(b)$,
\begin{align*}
\E[(g(X) - \alpha)^2] \geq \frac{(g(b) - \alpha)^2 \Var[X]}{(b - a)^2}\,.
\end{align*}
\end{lemma}

\begin{proof}
If it exists, let $x \in [a, b]$ be a point such that $g(x) = \alpha$. Otherwise let $x = a$, which by continuity of $g$ means that $g(a) < \alpha$.
Then let $\varphi : [a,b] \to \R$ be the linear function with $\varphi(b) = g(b)$ and $\varphi(x) = g(x)$, which satisfies
\begin{align*}
\E[(g(X) - \alpha)^2] 
&\geq \E[(\varphi(X) - \alpha)^2] \\
&\geq \min_{\eta \in \R} \E[(\eta(X - b) + g(b) - \alpha)^2] \\
&= \frac{(\alpha - g(b))^2 \Var[X]}{\E[(b - X)^2]} \\
&\geq \frac{(\alpha - g(b))^2 \Var[X]}{(b - a)^2}\,,
\end{align*}
where the first inequality is geometrically obvious from the convexity of $g$ (\cref{fig:linear}) and the second follows by optimising over all
linear functions passing through $(b, g(b))$.
The last equality follows by optimising for $\eta$, while the finally inequality holds because $X$ is supported on $[a,b]$.
\end{proof}

\begin{figure}
\centering
\includegraphics[width=5cm]{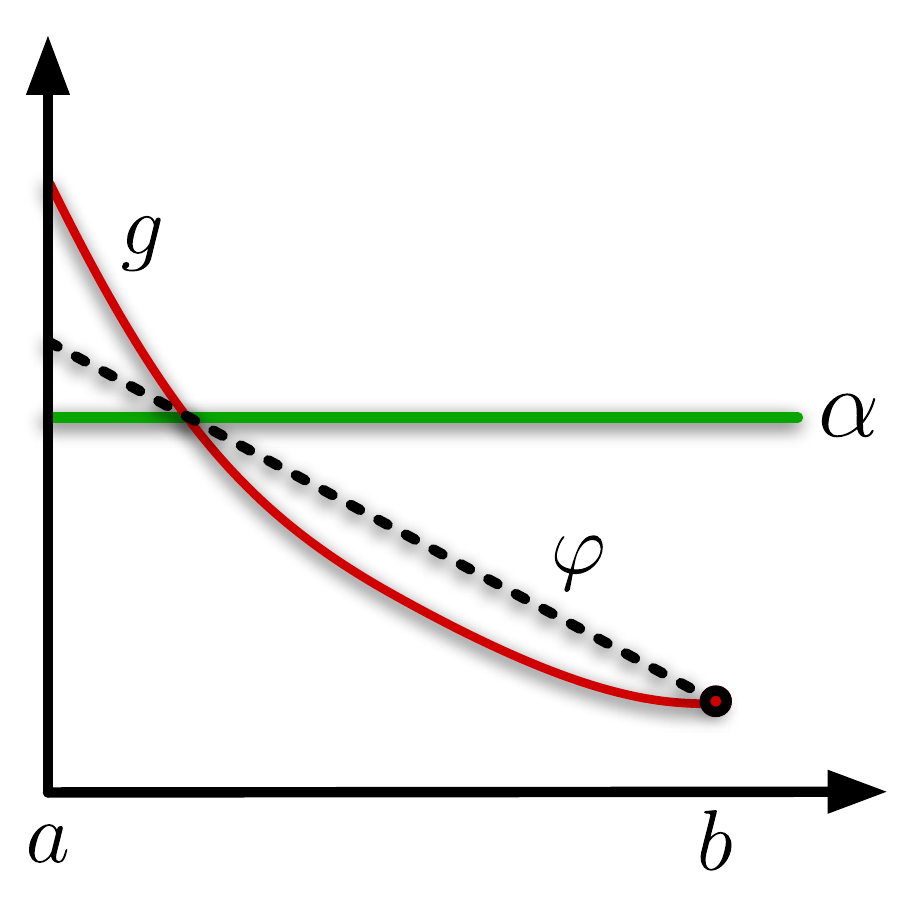}
\caption{The typical situation in the proof of Lemma~\ref{lem:var}. For all $x \in [a,b]$, 
the distance between the linear function $\varphi$ and the horizontal line at $\alpha$
is less than or equal to the distance between $f$ and the same horizontal line.}
\label{fig:linear}
\end{figure}

\begin{proof}[Proof of Proposition~\ref{prop:main}]
There are two steps. First, we define a collection of probability measures $\pi_0,\ldots,\pi_m$ with $m = O(\log(n))$.
In the second step we show this collection satisfies the conditions of \cref{thm:info}.

\paragraph{Step 1: Exploration distributions and perturbations}
Let $\bar f \in \conv(\cF)$ and
choose coordinates on $\cK$ so that $\zeros = \argmin_{x \in \cK} \bar f(x)$.
Let
\begin{align*}
\tilde f(x) = \bar f(x) + \frac{\norm{x}}{n \diam(\cK)} \,,
\end{align*}
which is close to $\bar f$, minimised at $\zeros$ and increases at least linearly with $\norm{x}$.
For $\epsilon > 0$, define
\begin{align*}
\cL_\epsilon = \left\{x \in \cK : \tilde f(x) = \tilde f_\star + \epsilon\right\}\,,
\end{align*}
which is a level set of $\tilde f$. 
Note that $\cL_\epsilon$ is not convex and boundary effects mean that $\zeros \in \conv(\cL_\epsilon)$ is not guaranteed.
Let $\pi_0$ be a Dirac at $\zeros$ and for $k \geq 1$ let
\begin{align*}
\epsilon_k = \frac{2^{k-1}}{\diam(\cK) n^2}
\end{align*}
and $\pi_k$ be the distribution on $\cL_k = \cL_{\epsilon_k}$
given by Lemma~\ref{lem:kw}. Consequentially, if $h_k$ is the support function of $\cL_k$, then 
for all $\theta \in \BS^{d-1}$,
\begin{align}
h_k(\theta) + h_k(-\theta) 
\leq 2 \sqrt{d \, \Var_{x \sim \pi_k}[\ip{x, \theta}]}\,.
\label{eq:pi}
\end{align}
Let 
\begin{align}
m = \max\{k : \epsilon_k \leq 1\} \leq 1 + \log_2(\diam(\cK)n^2) \,.
\label{eq:m}
\end{align}

\paragraph{Step 2}
Let $f \in \cF$ so that $f(x) = g(\ip{x, \theta})$ for some convex $g$ and $\theta \in \BS^{d-1}$.
Let $x_\star \in \cK$ be a minimiser of $f$ and assume without loss of generality that the sign of $\theta$ has been chosen so that
$\ip{x_\star, \theta} \geq 0$.
We will prove that there exists a $\pi \in \{\pi_0, \ldots, \pi_m\}$ such that 
\begin{align}
\int_{\cK} \tilde f \d{\pi} - f_\star \leq \frac{2}{n} + 16 \sqrt{d\int_{\cK} (\tilde f - f)^2 \d{\pi}}\,.
\label{eq:exp}
\end{align}
Let $\epsilon = \min\{\tilde f(x) - \tilde f_\star : x \in \cK, \ip{x - x_\star, \theta} = 0\}$,
which exists because $\tilde f$ is continuous and $\cK \cap \{x : \ip{x - x_\star, \theta} = 0\}$ is compact.
We claim that $h_{\cL_\epsilon}(\theta) = \ip{x_\star, \theta}$.
That $h_{\cL_\epsilon}(\theta) \geq \ip{x_\star, \theta}$ is immediate.
Suppose there exists an $x \in \cL_\epsilon$ with $\ip{x, \theta} > \ip{x_\star,\theta}$.
Then by convexity of $\cK$ and $\tilde f$, there exists a $y \in \cK$ with $\ip{y - x_\star, \theta} = 0$ and $\tilde f(y) < \tilde f_\star + \epsilon$, which
contradicts the definition of $\epsilon$.
Therefore $h_{\cL_\epsilon}(\theta) \leq \ip{x_\star, \theta}$ and hence $h_{\cL_\epsilon}(\theta) = \ip{x_\star, \theta}$.
Next, by the definition of $\tilde f$,
\begin{align*}
\cL_\epsilon \subset \{x : \tilde f(x) \leq \tilde f_\star + \epsilon\}
\subset \{x : \norm{x} \leq \epsilon n \diam(\cK)\}\,.
\end{align*}
Hence,
\begin{align}
\ip{x_\star, \theta} = h_{\cL_\epsilon}(\theta) \leq \epsilon n \diam(\cK)\,. 
\label{eq:ball}
\end{align}
Suppose for a moment that $\ip{x_\star, \theta} \geq 1/n$. It follows from \cref{eq:ball} that $\epsilon \geq 1/(\diam(\cK)n^2)$, in which
case there exists a largest $k$ such that $\epsilon_k \leq \epsilon$.
Let $x \in \cL_\epsilon$ be such that $\ip{x, \theta} = \ip{x_\star, \theta}$, which satisfies
$\tilde f(x) = \epsilon + \tilde f_\star$. By continuity, there exists an $\alpha \in [0,1]$ 
such that $\tilde f(\alpha x) = \tilde f_\star + \epsilon_k$. 
By convexity of $\tilde f$ and the fact that $\tilde f$ is minimised at $\zeros$, 
\begin{align*}
\epsilon_k + \tilde f_\star 
= \tilde f(\alpha x) 
\leq \alpha \tilde f(x) + (1 - \alpha) \tilde f_\star 
= \alpha \epsilon + \tilde f_\star\,.
\end{align*}
Rearranging shows that $\alpha \geq \epsilon_k / \epsilon \geq 1/2$ and hence 
\begin{align}
\ip{x_\star, \theta} \geq h_k(\theta) \geq \ip{\alpha x, \theta} \geq \frac{\ip{x_\star, \theta}}{2}
\label{eq:h-bound}
\end{align}
We consider four cases, the last two of which are illustrated in \cref{fig:34}:
\begin{enumerate}
\item $f_\star \abovelabel{1a}\geq \tilde f_\star - 2/n$.
\item $f_\star \abovelabel{2a}\leq \tilde f_\star - 2/n$ and $\ip{x_\star, \theta} \abovelabel{2b}\leq 1/n$.
\item $f_\star \abovelabel{3a}\leq \tilde f_\star - 2/n$ and $\ip{x_\star, \theta} \abovelabel{3b}\geq 1/n$ and $-h_k(-\theta) \abovelabel{3c}\geq h_k(\theta) / 2$.
\item $f_\star \abovelabel{4a}\leq \tilde f_\star - 2/n$ and $\ip{x_\star, \theta} \abovelabel{4b}\geq 1/n$ and $-h_k(-\theta) \abovelabel{4c}\leq h_k(\theta) / 2$.
\end{enumerate}
\paragraph{Case 1}
By (1a), \cref{eq:exp} holds with $\pi = \pi_0$ trivially.

\paragraph{Case 2}
By Fact~\ref{fact}, $g$ is Lipschitz on the closed interval $\{\ip{x, \theta} : x \in \cK\}$.
Combining this with (2a) and (2b) yields
\begin{align*}
\sqrt{\int_{\cK} (\tilde f - f)^2 \d{\pi}_0}
&=\tilde f_\star - f(\zeros) \\ 
\tag{$g$ is Lipschitz} &\geq \tilde f_\star - f_\star - \ip{x_\star, \theta} \\
&= \frac{1}{2} (\tilde f_\star - f_\star) + \frac{1}{2} (\tilde f_\star - f_\star) - \ip{x_\star, \theta} \\
\tag{using 2a and 2b} &\geq \frac{1}{2} (\tilde f_\star - f_\star) \\ 
&= \frac{1}{2}\left(\int_\cK \tilde f \d{\pi_0} - f_\star\right)\,.
\end{align*}
And again, \cref{eq:exp} holds with $\pi = \pi_0$.

\begin{figure}
\includegraphics[width=6cm]{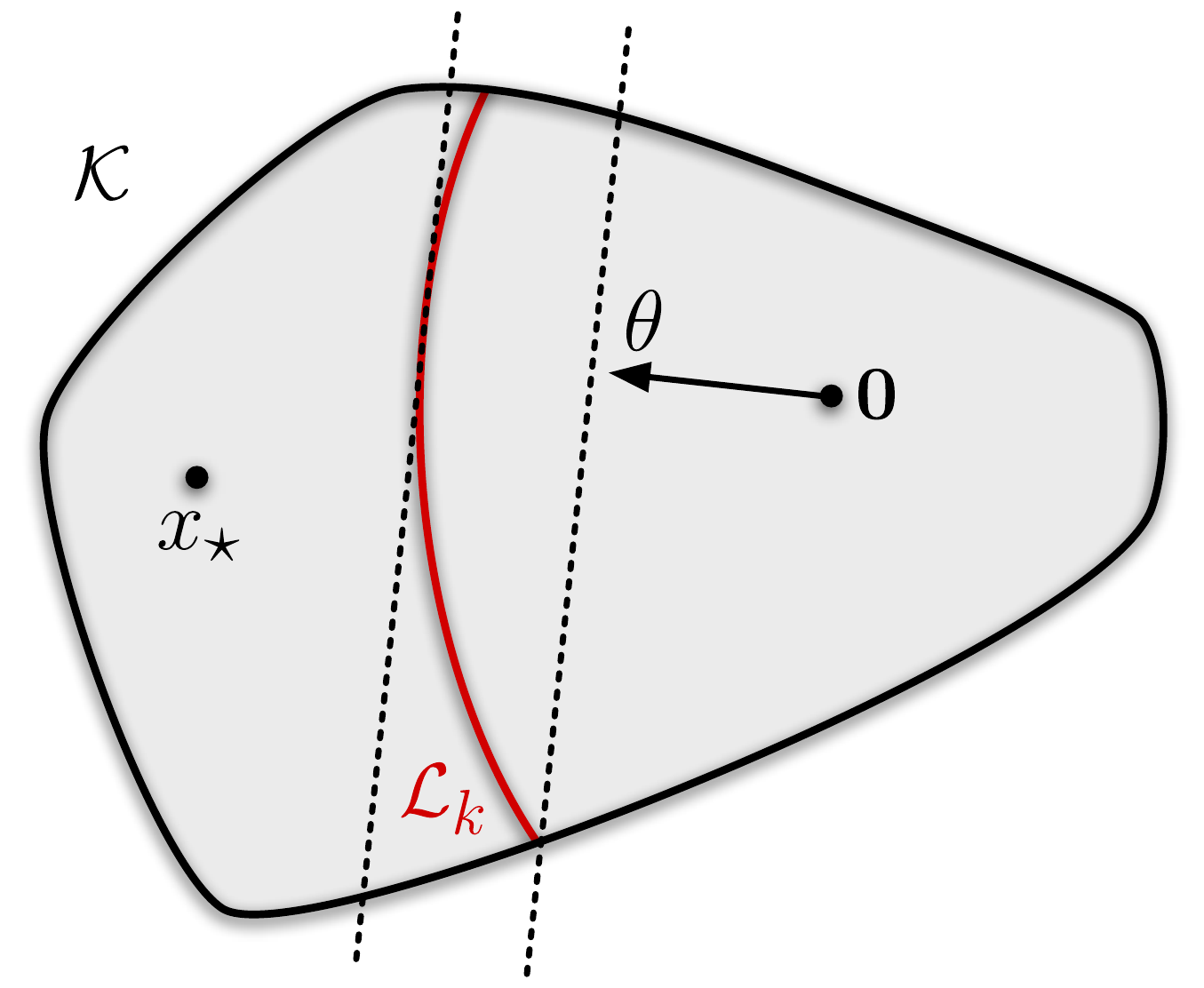}
\hspace{1cm}
\includegraphics[width=6cm]{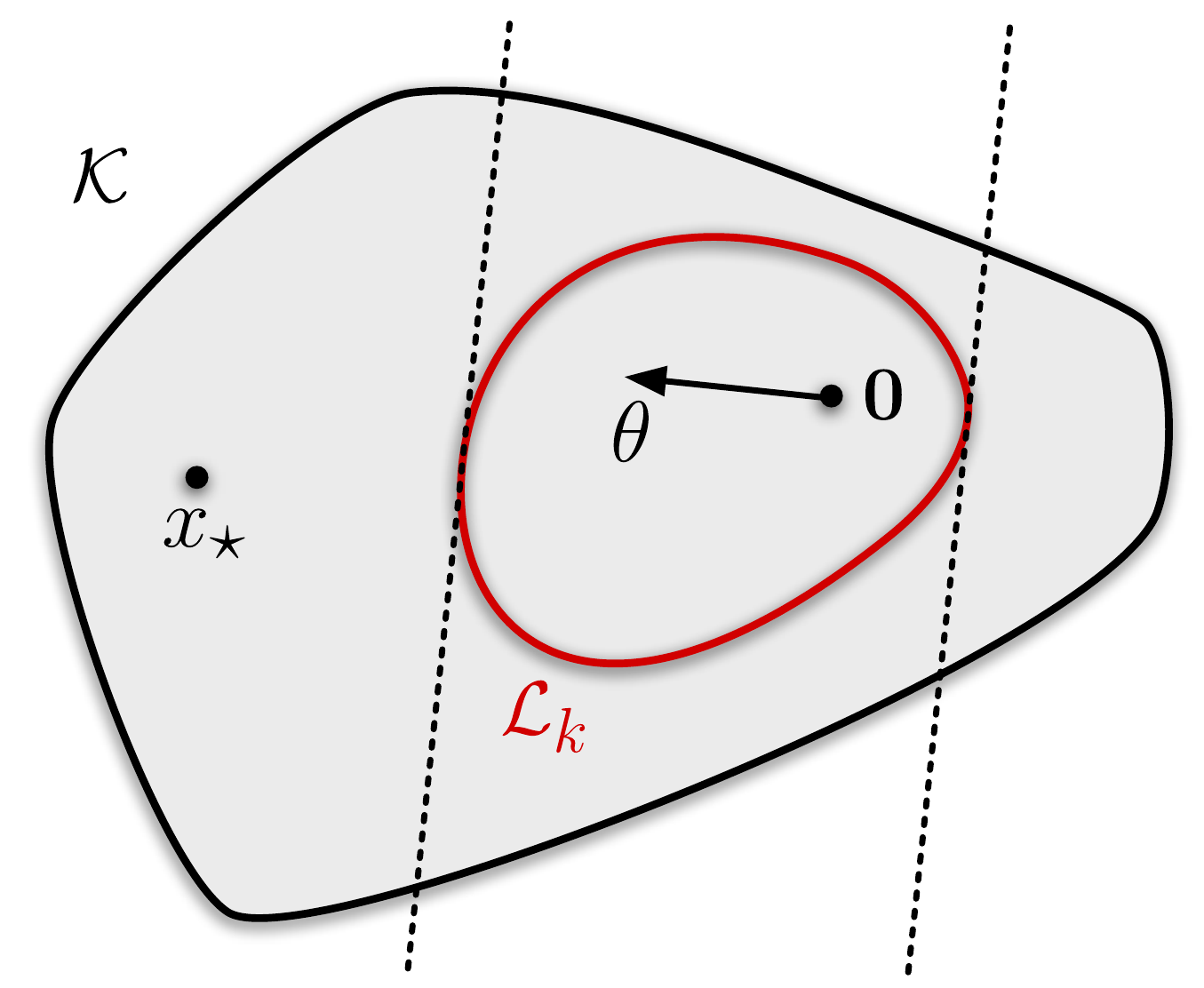}
\caption{
Examples of Case 3 (left) and Case 4 (right).
Case 3 only occurs when $-h_k(-\theta) \geq h_k(\theta) / 2$, which means that $\zeros$ is far outside
the convex hull of the level set $\cL_k$.
Meanwhile, in Case 4 the width of $\cL_k$ in the direction $\theta$ is at least the same order of magnitude as $h_k(\theta)$.
The distance between the dotted lines in both illustrations is the width of $\cL_k$ in direction $\theta$, which is $h_k(\theta) + h_k(-\theta)$.
Note that because $f$ is a ridge function, all points $x$ in $\cK$ with $\ip{x - x_\star, \theta} = 0$ minimise $f$.
}\label{fig:34}
\end{figure}

\paragraph{Case 3}
Let $\Delta = \tilde f_\star - f_\star$.
Suppose that $f(\zeros) \geq \tilde f_\star + \Delta / 4$. Then 
\begin{align*}
\int_{\cK} \tilde f \d{\pi_0} - f_\star
&= \Delta 
\leq 4 \sqrt{\int_{\cK} (\tilde f - f)^2 \d{\pi_0}}
\end{align*}
and \cref{eq:exp} holds with $\pi = \pi_0$.
Suppose for the remainder of this case that $f(\zeros) \leq \tilde f_\star + \Delta / 4$.  For all $x \in \cL_k$, 
\begin{align}
\ip{x_\star, \theta} 
\abovelabel{(by \cref{eq:h-bound})}\geq h_k(\theta) 
\geq \ip{x, \theta} 
\geq -h_k(-\theta) 
\abovelabel{(by 3c)}\geq h_k(\theta) / 2 
\abovelabel{(by \cref{eq:h-bound})}\geq \frac{\ip{x_\star, \theta}}{4} \,.
\label{eq:l-bound}
\end{align}
Therefore, $\ip{x, \theta} / \ip{x, \theta_\star} \in [0,1]$ and
\begin{align*}
f(x) 
&= f\left(\frac{\ip{x, \theta}}{\ip{x, \theta_\star}} x_\star + \left(1 - \frac{x, \theta}{\ip{x, \theta_\star}}\right) \zeros\right) \\
\tag{convexity of $g$} &\leq \frac{\ip{x,\theta}}{\ip{x_\star, \theta}} f_\star + \left(1 - \frac{\ip{x, \theta}}{\ip{x_\star,\theta}}\right) 
(\tilde f_\star + \Delta / 4) \\
&= \tilde f_\star + \frac{\Delta}{4}\left(1 - \frac{5\ip{x, \theta}}{\ip{x_\star, \theta}}\right) \\
\tag{by \cref{eq:l-bound}} &\leq \tilde f_\star - \frac{\Delta}{16}\,.
\end{align*}
Therefore, since $\pi_k$ is supported on $\cL_k$, 
\begin{align*}
\int_{\cK} (\tilde f - f)^2 \d{\pi_k}
= \int_{\cK} (\tilde f_\star + \epsilon_k - f)^2 \d{\pi_k}
\geq \frac{1}{16^2} (\Delta + \epsilon_k)^2
= \frac{1}{16^2} \left(\int_\cK \tilde f \d{\pi_k} - f_\star\right)^2\,. 
\end{align*}
Hence, \cref{eq:exp} holds with $\pi = \pi_k$.

\paragraph{Case 4}
By definition of the support function $\ip{x, \theta} \in [-h_k(-\theta), h_k(\theta)]$ for all $x \in \cL_k$.
Furthermore, by \cref{eq:h-bound}, $\ip{x_\star, \theta} \in [h_k(\theta), 2 h_k(\theta)]$.
Hence, by Lemma~\ref{lem:var} with $a = -h_k(-\theta)$, $b = \ip{x_\star, \theta} \leq 2 h_k(\theta)$ and $\alpha = \tilde f_\star + \epsilon_k$,
\begin{align*}
\int_\cK (\tilde f - f)^2 \d{\pi_k}
\tag{$\pi_k$ supported on $\cL_k$} &= \int_\cK (\tilde f_\star + \epsilon_k - g(\ip{x, \theta}))^2 \d{\pi_k}(x) \\
\tag{By Lemma~\ref{lem:var}} &\geq \frac{(\tilde f_\star + \epsilon_k - f_\star)^2 \Var_{x \sim \pi_k}[\ip{x, \theta}]}{(2 h_k(\theta) + h_k(-\theta))^2} \\
\tag{By \cref{eq:pi}} &\geq \frac{(\tilde f_\star + \epsilon_k - f_\star)^2 (h_k(\theta) + h_k(-\theta))^2}{4d (2 h_k(\theta) + h_k(-\theta))^2} \\
&\tag{By 4c}\geq \frac{(\tilde f_\star + \epsilon_k - f_\star)^2}{64 d} \\
&\tag{$\pi_k$ supported on $\cL_k$}= \frac{1}{64d} \left(\int_\cK \tilde f \d{\pi_k} - f_\star\right)^2 \,.
\end{align*}
where the final inequality is trivial if $h_k(-\theta)$ is positive, while if $h_k(-\theta)$ is negative, then using (4c),
$h_k(\theta) + h_k(-\theta) \geq h_k(\theta)/2 \geq (2 h_k(\theta) + h_k(-\theta)) / 4$ and the result follows.

\paragraph{Summary}
We have shown that for all $f \in \cF$ there exists a policy $\pi \in \{\pi_0, \ldots, \pi_m\}$ such that \cref{eq:exp}
holds. By the definition of $\tilde f$ it follows that
\begin{align*}
\int_\cK \bar f \d{\pi} - f_\star 
\tag{since $\bar f \leq \tilde f$}&\leq \int_\cK \tilde f \d{\pi} - f_\star \\ 
\tag{by \cref{eq:exp}} &\leq \frac{2}{n} + 16 \sqrt{d \int_\cK (\tilde f - f)^2 \d{\pi}} \\
&\leq \frac{2 + 16 \sqrt{2d}}{n} + 16 \sqrt{2d \int_\cK (\bar f - f)^2 \d{\pi}} \,,
\end{align*}
where we used the inequalities $(a + b)^2 \leq 2a^2 + 2b^2$ and $\sqrt{a + b} \leq \sqrt{a} + \sqrt{b}$ for reals $a, b \geq 0$.
The result now follows from \cref{eq:m} and \cref{thm:info} and noting that in this proof the number of exploration policies is $m+1$.
\end{proof}

\section{Proof of Fact~\ref{fact}}\label{sec:fact}
\newcommand{\Z}{\mathbb Z}
Convexity is obvious.
We need to prove that $g$ is Lipschitz on $\{\ip{x, \theta} : x \in \cK\}$.
Elementary differentiation does not suffice: for example, if $d = 2$ and $f(x) = \sqrt{2} x_1$ and $\cK = \{(x, x) : x \in [0,1]\} \subset \R^2$,
then $f$ is Lipschitz on $\cK$, but $g(x) = \sqrt{2} x$ is not Lipschitz. The problem is that $\cK$ is not a convex body because it does
not have a non-empty interior. 
To exploit the fact that $\cK$ is a convex body, note that for any $x, y$ in the interior of $\cK$ with $\ip{y - x, \theta} \geq 0$, 
there exists a path
from $x$ to $y$ that moves either perpendicular to $\theta$ or in the direction of $\theta$.
Since $f$ is constant on hyperplanes $\{z : \ip{z, \theta} = c\}$ for any $c$, the Lipschitzness of $f$ applied along the path
implies that $g$ is also Lipschitz.
The extension to $x, y$ on the boundary of $\cK$ follows from the continuity of $f$ and by passing to the limit.

\newcommand{\cC}{\mathcal C}
\newcommand{\sP}{\mathcal P}
\newcommand{\cN}{\mathcal N}
\newcommand{\cG}{\mathcal G}
\newcommand{\cKa}{\mathcal K_{\textrm{a}}}
\newcommand{\BSa}{\BS^{d-1}_{\textrm{a}}}
\newcommand{\cHa}{\cH_{\textrm{a}}}
\newcommand{\poly}{\operatorname{poly}}

\section{Proof outline of \cref{thm:info}}\label{sec:info-app}
The majority of the argument follows that given by \cite[Appendix A]{Lat20-cvx}.
Let $\cKa \subset \cK$ and $\BSa \subset \BS^{d-1}$ be finite covering sets such that 
\begin{align*}
\sup_{x \in \cK} \inf_{y \in \cKa} \norm{x - y} \leq \epsilon 
\qquad \text{and} \qquad \sup_{x \in \BS^{d-1}} \inf_{y \in \BSa} \norm{x - y} \leq \epsilon\,,
\end{align*}
with $\epsilon = \poly(1/n)$ suitably small. 
Let $\cH = \cup_{\theta \in \BS^{d-1}} \cF_\theta^n$, which is the set of
possible sequences of functions that the adversary can choose. Let $\cHa = \cup_{\theta \in \BSa} \cF_\theta^n$.
We assume that rather than observe $f_t(x_t)$ after round $t$, the learner observes $y_t = f_t(x_t) + \eta_t$ where $\eta_t$ is a Gaussian with zero mean
and variance $\sigma^2 = \poly(1/n)$ and truncated in $[-1,1]$. This can only increase the regret, since the effect can be modelled by restricting
the class of policies to add the noise before making a decision.
The value of $\sigma^2$ should be chosen such that the probability that a truncation occurs is $O(1/n^2)$ while $\epsilon$ should be $O(\poly(\sigma^2/n))$.
Standard results show that $\log |\cKa|$ and $\log |\BSa|$ are both at most $\const d \log(n\diam(\cK))$.

\paragraph{Policies and the Bayesian probability space}
Let $\sP(A)$ be the space of finitely supported probability measures on set $A$ with the discrete $\sigma$-algebra.
Given a measure $\nu \in \sP(\cH)$ and a policy $\pi$, let $\bbP_\nu^\pi$ be the probability measure on $(f_t)_{t=1}^n$ and $(x_t)_{t=1}^n$ and $(y_t)_{t=1}^n$
when $(f_t)_{t=1}^n$ are sampled from $\nu$ and the laws of the outcome are determined by the interaction between $\pi$ and the bandit.
Expectations with respect to $\bbP_\nu^\pi$ are denoted by $\E_\nu^\pi$.
The optimal action in $\cKa$ in hindsight is $x_\star = \argmin_{x \in \cKa} \sum_{t=1}^n f_t(x)$.

\paragraph{Minimax duality}
Let $\Pi_\textrm{a}$ be the set of policies that plays actions in $\cKa$ almost surely.
Let $\Reg_n(\pi, (f_t)_{t=1}^n)$ be the regret for a given policy and sequence of functions.
By minimax duality,
\begin{align}
\Reg_n
&\leq \inf_{\pi \in \Pi_{\textrm{a}}} \sup_{(f_t)_{t=1}^n \in \cH} \Reg_n(\pi, (f_t)_{t=1}^n) \nonumber \\
&\leq 1 + \sup_{\nu \in \sP(\cH)} \inf_{\pi \in \Pi_\textrm{a}} \E_\nu^\pi\left[\sum_{t=1}^n f_t(x_t) - f_t(x_\star)\right] \nonumber \\
&\leq 2 + \sup_{\nu \in \sP(\cHa)} \inf_{\pi \in \Pi_\textrm{a}} \E_\nu^\pi\left[\sum_{t=1}^n f_t(x_t) - f_t(x_\star)\right]\,, \label{eq:bayes}
\end{align}
where the first inequality is trivial. The second follows from minimax duality \citep{LS19pminfo}. 
The last inequality follows by choosing $\epsilon = \poly(1/n)$ suitably small and using standard information-theoretic arguments that for all $(f_t)_{t=1}^n \in \cH$
there exists an approximation in $\cHa$ that because of the noisy losses is statistically indistinguishable from the perspective of the policy.
This is where the conditions on $\epsilon$ and $\sigma^2$ are needed, but we omit details.

\paragraph{Bayesian regret}
For the remainder we bound \cref{eq:bayes} for any $\nu \in \sP(\cHa)$ and a carefully constructed policy $\pi$. Abbreviate $\E = \E_\nu^\pi$. 
Let $\theta \in \BSa$ be the random element such that $f_t(x) = g_t(\ip{x, \theta})$ for suitably chosen $(g_t)_{t=1}^n$ and let $Z = (x_\star, \theta)$.
The next step is to bound the Bayesian regret on the right-hand side of \cref{eq:bayes}. 
Let $\bbP_t(\cdot) = \bbP(\cdot | x_1, y_1,\ldots,x_t,y_t)$ and $\E_t$ be the expectation with respect to $\bbP_t$.
For $z \in \cKa \times \BSa$, let
\begin{align*}
f_{t,z}(x) = \E_{t-1}[f_t(x) | Z = z] \qquad \text{and} \qquad \bar f_t(x) = \E_{t-1}[f_t(x)] \,.
\end{align*}
Note that $f_{t,z}$ is a ridge function.
Let $\mu_t$ be the finitely supported measure on $\cF$ with $\mu_t(\{f_{t,z}\}) = \bbP_{t-1}(Z = z)$.
In previous arguments $Z$ was defined to be the optimal action, but in our setup the resulting function would no longer be a ridge function and the analysis
in \cref{sec:main} would not apply.
By the assumptions of the theorem and a lemma for combining exploratory distributions \citep[Lemma 4]{Lat20-cvx}, there exists a distribution $\pi_t$ on $\cKa$ such that
\begin{align*}
\int_{\cKa} \bar f_t(x) \d{\pi}_t(x) - \int_{\cF} f_\star \d{\mu_t}(f) \leq 1/n + \alpha + \sqrt{\beta m \int_{\cKa} \int_\cF (\bar f_t - f)^2 \d{\mu_t}(f) \d{\pi_t}(x)}\,.
\end{align*}
The additional constant $1/n$ appears because $\pi_t$ must modified to be supported on $\cKa$.
We let $\pi = (\pi_t)_{t=1}^n$.
Then,
\begin{align*}
\E\left[\sum_{t=1}^n f_t(x_t) - f_t(x_\star)\right]
&\leq \E\left[\sum_{t=1}^n \left(\int_{\cKa} \bar f_t(x) \d{\pi_t}(x) - \int_{\cF} f_\star \d{\mu_t}(f)\right)\right] \\
&\leq 1 + n\alpha + \E\left[\sum_{t=1}^n \sqrt{\beta m \int_{\cKa} \int_\cF (\bar f - f)^2 \d{\mu_t} \d{\pi_t} }\right]\,.
\end{align*}
Next, we relate the right-hand side above to the information gain about $Z$. 
Letting $I_t(U ; V)$ be the mutual information between random elements $U$ and $V$ under $\bbP_t$,
by Pinsker's inequality,
\begin{align*}
\int_\cF \int_\cK (\bar f - f)^2 \d{\pi_t} \d{\mu_t}
&= \E_{t-1}\left[(\E_{t-1}[y_t|x_t] - \E_{t-1}[y_t | Z,x_t])^2\right] \\
&\leq \E_{t-1}\left[I_{t-1}(Z ; x_t, y_t)\right]\,. 
\end{align*}
By the chain rule for the mutual information and letting $H(Z)$ be the entropy of random element $Z$,
\begin{align*}
\E\left[\sum_{t=1}^n I_{t-1}(Z ; x_t, y_t)\right] \leq H(Z) \leq \log |\cKa| + \log |\BSa| \leq \const d \log(n \diam(\cK))\,.
\end{align*}
Therefore, by Cauchy--Schwarz, the Bayesian regret is bounded by
\begin{align*}
\E\left[\sum_{t=1}^n f_t(x_t) - f_t(x_\star)\right]
&\leq 1 + n\alpha + \E\left[\sum_{t=1}^n \sqrt{\beta m I_{t-1}(Z ; x_t, y_t)}\right] \\
&\leq 1 + n\alpha + \sqrt{\beta m n \E\left[\sum_{t=1}^n I_{t-1}(Z ; x_t, y_t)\right]} \\
&\leq 1 + n\alpha + \const \sqrt{\beta m n d \log(n\diam(\cK))}\,.
\end{align*}

\bibliographystyle{plainnat}
\bibliography{all}

\end{document}